\documentclass[11pt, a4paper, onecolumn, confidential, copyright, goog]{google}

\usepackage[authoryear, sort&compress, round]{natbib}
\usepackage{microtype}
\usepackage{graphicx}
\usepackage{subfigure}
\usepackage{booktabs}
\usepackage{pifont} 
\usepackage{mdframed}
\usepackage{xcolor}
\usepackage{tabularray}
\usepackage{fontawesome5}
\usepackage[normalem]{ulem}
\usepackage[utf8]{inputenc}
\usepackage[T1]{fontenc}
\usepackage{hyperref}
\usepackage{url}
\usepackage{booktabs}
\usepackage{amsfonts}
\usepackage{nicefrac}
\usepackage{microtype}
\usepackage{xcolor}
\usepackage{amsmath}
\usepackage{amssymb}
\usepackage{amsthm}
\usepackage{graphicx}
\usepackage{enumitem}
\usepackage{mathtools}
\usepackage{algorithm}
\usepackage{algpseudocode}
\usepackage[most]{tcolorbox}
\usepackage{subcaption}
\usepackage{wrapfig}
\usepackage[utf8]{inputenc}
\usepackage{amsmath}
\usepackage{amssymb}
\usepackage{amsthm}
\usepackage{mathtools}
\usepackage{geometry}
\usepackage{algorithm}
\usepackage{algpseudocode}
\usepackage{xcolor}
\usepackage{caption}
\theoremstyle{definition}

\newtheorem{proposition}{Proposition}[section]

\newtheorem{remark}{Remark}[section]
\usepackage{hyperref}
\usepackage{xspace}
\usepackage{booktabs}
\usepackage{multirow}
\usepackage{tabularray}
\usepackage{adjustbox}
\usepackage{framed}
\usepackage{dsfont}
\usepackage{wrapfig}

\newcommand{\baseg}{G_{\mathrm{base}}}
\newcommand{\reportbank}{\mathcal{B}}
\newcommand{\broadmem}{\mathcal{M}}
\newcommand{\localmem}{\mathcal{L}}
\newcommand{\allmem}{\mathcal{Q}}

\newcommand{\conf}{\mathrm{Conf}}
\newcommand{\thref}{\tau_{\mathrm{refuse}}}
\newcommand{\thallow}{\tau_{\mathrm{allow}}}
\newcommand{\gfun}{g}

\uselogo{} 

\title{LiSA: Lifelong Safety Adaptation via Conservative Policy Induction}

\correspondingauthor{Minbeom Kim: minbeomkim@snu.ac.kr and Long T. Le: longtle@google.com}

\reportnumber{} % Leave blank if n/a

\renewcommand{\today}{}

\author[1 2 *]{Minbeom Kim}
\author[1]{Lesly Miculicich}
\author[1]{Bhavana Dalvi Mishra}
\author[1]{Mihir Parmar}
\author[3]{Phillip Wallis}
\author[3]{Bharath Chandrasekhar}
\author[2]{Kyomin Jung} 
\author[1]{Tomas Pfister}
\author[1]{Long T. Le}

\affil[1]{Google Cloud AI Research}
\affil[2]{Seoul National University}
\affil[3]{Google}

\begin{abstract}

As AI agents move from chat interfaces to autonomous systems that read private data, call tools, and execute multi-step workflows, guardrails become an important line of defense against concrete deployment harms. In these settings, guardrail failures are no longer merely answer-quality errors: they can leak secrets, authorize unsafe actions, or block legitimate work. The hardest failures are often contextual: whether an action is acceptable depends on local privacy norms, organizational policies, and user expectations that resist pre-deployment specification. This creates a practical gap: guardrails must adapt to their own operating environments, yet deployment feedback is typically limited to sparse, noisy user-reported failures, and repeated fine-tuning is often impractical.
To address this gap, we propose \textsc{LiSA} (\textbf{Li}felong \textbf{S}afety \textbf{A}daptation), a conservative policy induction framework that improves a fixed base guardrail through structured memory. \textsc{LiSA} converts occasional failures into reusable policy abstractions so that sparse reports can generalize beyond individual cases, adds conflict-aware local rules to prevent overgeneralization in mixed-label contexts, and applies evidence-aware confidence gating via a posterior lower bound, so that memory reuse scales with accumulated evidence rather than empirical accuracy alone. Across PrivacyLens+, ConFaide+, and AgentHarm, \textsc{LiSA} consistently outperforms strong memory-based baselines under sparse feedback, remains robust under noisy user feedback even at 20\% label-flip rates, and pushes the latency--performance frontier beyond backbone model scaling. Ultimately, \textsc{LiSA} offers a practical path to secure AI agents against the unpredictable long tail of real-world edge risks.

\end{abstract}

\begin{document}
\maketitle

\section{Introduction}

Large language models (LLMs) increasingly power AI agents that do more than answer questions: they access private data~\citep{abaev2026agentguardian}, call privileged tools~\citep{workarena}, and execute multi-step workflows~\citep{taubench}. As such systems move into deployment, the cost of error escalates from low-stakes generation mistakes to concrete harms: incorrect allow decisions can leak private information or authorize unsafe actions, while incorrect refusals can block legitimate work. To mitigate these risks, deployed AI systems increasingly rely on safety guardrails as a critical last line of defense.

A growing line of work~\citep{agrail, llamaguard, shieldgemma, causalarmor, piguard} has introduced different guardrails, including refusal-oriented prompting, safety classifiers, rule-based validators, and runtime monitors. However, these methods share a fundamental limitation: they rely on a static, general-purpose definition of harm specified before deployment. In practice, safety and privacy boundaries are rarely universal. Acceptable behavior is shaped by local organizational rules, shifting user expectations, and task-specific risk tolerances that are difficult to fully enumerate in advance~\citep{privacyreasoning, contextual1, contextual2}. Consequently, a fixed guardrail is often mismatched to its unique deployment environment—leaving it too permissive against novel risks while remaining too restrictive for legitimate, context-specific actions.

To bridge this gap, we formulate the problem of \textbf{deployment-time guardrail adaptation}: a deployed guardrail should improve over time from the failures that arise in its own operating context. This setting imposes three constraints that distinguish it from standard supervised updating. First, adaptation must occur under \emph{sparse supervision}~\citep{sparse}: users rarely provide dense, curated labels, yielding only occasional corrections. Second, feedback can be \emph{noisy}~\citep{noisy}: users may disagree, misattribute failures, or report preferences as safety concerns. Third, adaptation must remain \emph{conservative}~\citep{conservatism}: overgeneralizing from a handful of local mistakes can degrade helpfulness through overly broad refusals, or compromise safety by over-trusting weakly supported permissive rules.

To address these challenges, we propose \textbf{\textsc{LiSA}} (\textbf{Li}felong \textbf{S}afety \textbf{A}daptation), a conservative policy induction framework organized as an online--offline loop. Rather than repeatedly fine-tuning the base guardrail, \textsc{LiSA} improves it through structured policy memory and evidence-aware reuse as in Figure~\ref{fig:overview}. Broad policy abstractions turn sparse failure reports into reusable guidance; conflict-aware local policies preserve fine-grained resolution near mixed-label regions; and confidence-gated reuse surfaces broad memory only when accumulated evidence supports it, preventing weakly supported abstractions from influencing inference too early. Together, these components allow a fixed guardrail to adapt to its operating environment while remaining stable under sparse, noisy feedback.

Empirically, we evaluate \textsc{LiSA} on PrivacyLens+~\citep{privacylens}, ConFaide+~\citep{confaide}, and AgentHarm~\citep{agentharm} under simulated deployment streams with sparse failure reports. Across datasets and two lightweight online guardrails, \textsc{LiSA} consistently outperforms the fixed base guardrail and strong memory-based baselines. Ablations reveal that while broad abstraction aids adaptation—much like existing memory baselines—local policies drive the most substantial performance improvements. Furthermore, confidence gating stabilizes these gains and renders the system highly robust even when reported labels are noisy. Finally, our latency analysis shows that structured memory offers a more efficient path than simply scaling the guardrail: LISA attached to the lightweight model pushes the latency--performance frontier beyond larger un-adapted backbones.

\begin{figure}[t]
    \centering
    \includegraphics[trim=0 30 0 85, clip, width=\linewidth]{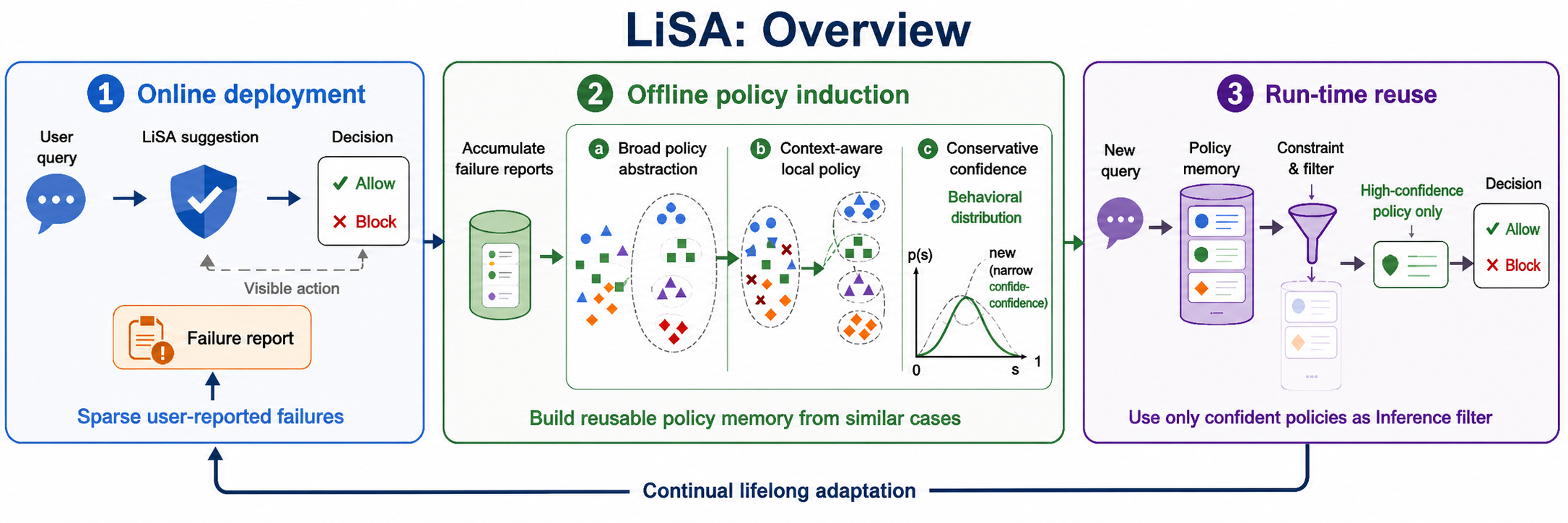}
    \caption{\textbf{Overview of \textsc{LiSA} guardrail}. \textit{1) Online (left)}: a fixed base guardrail decides on each query and logs user-reported mistakes. \textit{2) Offline (center)}: sparse reports are abstracted into broad policy items, while mixed-label neighborhoods are rendered into conflict-aware local refinement rules. Broad policy items are scored by a Beta posterior over their support/contradiction counts. \textit{3) Run-time reuse (right)}: semantically matched local rules are surfaced as narrow refinement cues, and broad policies are surfaced only when their posterior lower bound clears a label-specific threshold.}
    \label{fig:overview}
\end{figure}
\newpage
\paragraph{Contributions.}
Our contributions are threefold:
\begin{itemize}[leftmargin=1.4em]
    \item We formulate the problem of \textit{lifelong guardrail adaptation}, where a fixed base guardrail improves from sparse, potentially noisy user-reported corrections without repeated fine-tuning.

    \item We propose \textbf{\textsc{LiSA}}, a structured policy memory framework driven by three core mechanisms: broad policy abstraction for sparse reuse, conflict-aware local refinement for mixed-label regions, and conservative confidence-gated reuse that surfaces memory only when accumulated evidence warrants it.

    \item Across PrivacyLens+, ConFaide+, and AgentHarm, we demonstrate that \textsc{LiSA}: (i) consistently outperforms strong memory-based baselines under sparse feedback; (ii) remains robust to noisy reports, with conservative confidence-gated reuse identified as a key stabilizing factor; (iii) improves boundary-sensitive decisions through conflict-aware local refinement; and (iv) pushes the latency--performance frontier beyond base-model scaling, showing that memory-based adaptation can be a more efficient route to improving deployed guardrails than using larger static backbones.
\end{itemize}

\section{Related works}

\subsection{Guardrails for AI agent safety}

A broad family of guardrailing mechanisms has emerged as LLMs gain access to private data and privileged tools, including general-purpose safety classifiers~\citep{llamaguard, shieldgemma, guardreasoner, wildguard}, implicit toxicity detectors~\citep{lifetox}, refusal-oriented safeguards~\citep{refuse, refuse2}, monitors over agent reasoning or trajectories~\citep{cotmonitor}, and defenses against indirect prompt manipulation~\citep{causalarmor, piguard, guardagent, leng2025static}. While these methods address complementary risk surfaces, they are typically specified before deployment and remain largely fixed during use. When out-of-distribution failures emerge in real environments, a natural response is to collect more examples and update the guardrail through retraining~\citep{oodhandling}. This path is mismatched to the regime we study, where failures are sparse, feedback arrives as occasional user reports, and repeated training is often impractical~\citep{continualfail}. We therefore ask whether a fixed base guardrail can improve directly from deployment-time experience, without retraining, by organizing sparse feedback into lightweight reusable structure.

\subsection{Memory and policy abstraction for adaptive agents}

A growing body of work equips LLM agents with memory~\citep{hu2025memory} so that they can accumulate experience, either by retrieving past trajectories as exemplars~\citep{synapse} or by inducing reusable natural-language policies, codes~\citep{lee2026program}, or reflections from prior outcomes~\citep{reflexion, tan2025prospect, reasoningbank, reflectcap}. These methods are typically developed for reasoning and planning settings~\citep{hu2025memory, choi2026policybank}, where supervision comes from task success and a poorly surfaced memory item degrades answer quality rather than causing a direct safety failure; broad abstraction and relatively permissive retrieval are reasonable defaults in that regime.

Guardrailing departs from this regime in several ways that matter for memory design. Labels are contextual~\citep{privacyreasoning} and often user- or organization-specific~\citep{abaev2026agentguardian, orgaccess} rather than determined by task success, so induced policies inherit feedback noise directly. Moreover, a single mis-surfaced memory item can trigger a privacy leak or an unsafe allow decision, making weakly supported reuse much riskier than in task-oriented memory systems. These properties make broad abstraction useful but also make naive retrieval substantially more brittle in the guardrail setting.

Within safety guardrails domain, adaptive memory remains relatively underexplored. 
AGrail~\citep{agrail} maintains an updatable safety checklist, but checklist-style adaptation has limited resolution in mixed-label regions and does not explicitly calibrate memory reuse by confidence. 
Recent works study personalized guardrails~\citep{personalguard, personalsafety} by conditioning safety reasoning on user profiles, but focus on profile-conditioned decision making rather than learning from sparse case-level corrections. 
\textsc{LiSA} targets this complementary deployment-time adaptation setting by jointly addressing two failure modes of prior adaptive memory: it adds conflict-aware local refinement so that mixed-label neighborhoods are not collapsed into a single broad rule, and it gates broad-policy reuse with a Beta-posterior lower bound rather than relying on retrieval similarity or empirical accuracy alone.

\section{\textsc{LiSA} guardrails}

\subsection{Problem setup and deployment loop}

We study deployment as an alternating \emph{online--offline} loop. Online, the guardrail receives a stream of deployment inputs $x_t \in \mathcal{X}$,
and outputs a binary decision $\hat y_t \in \{0,1\}$,
where $0$ denotes \textsc{allow} and $1$ denotes \textsc{refuse}. A fixed base guardrail
\[
\baseg : \mathcal{X} \to \{0,1\}
\]
is available throughout deployment.
As the system is used, it accumulates sparse user-reported corrections
$
\reportbank_t = \{(x_i,\tilde y_i)\}_{i=1}^{n_t}.
$
These reports arrive irregularly and may be noisy. Rather than repeatedly fine-tuning the guardrail, we periodically refresh memory from the accumulated reports and redeploy the updated memory in the next online phase. This yields a lightweight form of \emph{lifelong safety adaptation}: the deployed guardrail improves over time while the base guardrail remains fixed.

Our method combines three components: broad policy memory for reusable coverage (\S\ref{subsec:policy}), conflict-aware local policies for ambiguous regions (\S\ref{subsec:local}), and confidence-gated reuse (\S\ref{subsec:gating}) so that broad abstractions are surfaced only when sufficiently supported, while local rules are used as narrow refinement cues for semantically close mixed-label cases.

\subsection{Structured policy memory}
\label{subsec:policy}

The central unit of adaptation is a \emph{policy item}. At each offline refresh, \textsc{LiSA} converts newly reported failures into broad policy candidates and merges semantically overlapping candidates across refreshes. A broad policy item is represented as
\[
m = (r_m,\ell_m,\nu_m),
\]
where $r_m$ is a natural-language policy statement, $\ell_m \in \{0,1\}$ is the label it recommends, and $\nu_m$ stores metadata such as provenance, examples, and runtime statistics. For instance, a broad item induced during deployment may read 
``\textit{Sharing general or public information is appropriate even by 
professionals in confidential roles},'' with $\ell_m=\textsc{allow}$ and $\nu_m$ 
aggregating support and contradiction counts across the reports that induced 
it (Appendix~\ref{app:policy-examples}, Example~1). These items are designed for sparse reuse: rather than storing each failure only as an isolated case, \textsc{LiSA} stores a compact abstraction that can guide future decisions in related contexts. The resulting broad memory is
\[
\broadmem_t = \{m_j\}_{j=1}^{M_t}.
\]

\paragraph{Why broad abstraction alone misses local boundaries.}
Broad memory improves reuse under sparse feedback, but it can also become too coarse. Nearby contexts with different labels may be covered by the same broad policy, causing the memory to overgeneralize across a local decision boundary. Since sparse feedback does not support refining every broad policy, \textsc{LiSA} adds conflict-aware local policies only in mixed-label regions where broad reuse is most likely to fail. Section~\ref{sec:formal-support} formalizes this motivation, and Appendix~\ref{app:offline-refresh} gives the operational refresh procedure.

\subsection{Conflict-aware local policies}
\label{subsec:local}
When the report neighborhood associated with a broad pattern contains both labels with non-trivial support, we treat that region as evidence that broad reuse is overgeneralizing across a local boundary. For instance, coworker-to-coworker sharing may be appropriate for routine coordination but inappropriate when it involves client insurance information without clear need-to-know authorization~\citep{nissenbaum2004privacy}.
We then induce one or more narrower policy items
\[
e = (r_e,\ell_e,\nu_e),
\]
and store them in local memory
\[
\localmem_t = \{e_k\}_{k=1}^{L_t}.
\]
As an instance, a region centered on ``a friend attended a public lecture'' splits between \textsc{allow} and \textsc{refuse} depending on whether the lecture is a public talk or a fringe event; \textsc{LiSA} renders complementary label-specific cues for this region rather than forcing a single broad rule across the boundary (Appendix~\ref{app:policy-examples}, Example~3). Broad and local policies are stored as natural-language memory entries rather than executable rules, but they play distinct roles and are governed by different reuse rules. Broad memory provides reusable coverage under sparse feedback and is therefore subject to confidence gating (Section~\ref{subsec:gating}), so that weakly supported abstractions do not influence inference too early. Local memory, by contrast, is induced \emph{only} in mixed-label regions where nearby cases split labels; its purpose is to expose a contradictory boundary cue to the inference model rather than to assert a globally reusable rule. Because a local rule is, by construction, anchored to a conflict-heavy semantic neighborhood, applying the same broad-policy gate to it would suppress exactly the boundary signal it is meant to surface. We therefore do \emph{not} gate local rules at inference time: any retrieved local rule is surfaced together with its support and contradiction counts as evidence for the inference model. Appendix~\ref{app:offline-refresh} specifies the deterministic procedure that detects mixed-label regions and renders label-specific local rules.

\subsection{Confidence-gated online guardrailing}
\label{subsec:gating}
Let
\[
\allmem_t = \broadmem_t \cup \localmem_t
\]
denote the full memory at deployment time. For each broad item $m \in \broadmem_t$, we maintain support and contradiction counts $(s_m, c_m)$, initialized from the inducing reports and updated when surfaced broad items later receive additional feedback. Local items also store support and contradiction counts, but these counts are serialized as local evidence rather than used for confidence gating.

We model the transfer reliability of a broad policy item by a latent accuracy $\theta_m \in [0,1]$, the probability that the item remains correct when surfaces on a future case. With a uniform prior,
\[
\theta_m \sim \mathrm{Beta}(1,1),
\]
the posterior after observing support and contradiction counts is
\[
\theta_m \mid \text{data} \sim \mathrm{Beta}(1+s_m,\,1+c_m).
\]
We define confidence as the lower $\delta$-quantile of this posterior,
\begin{equation}
\label{eq:beta-confidence}
\conf(m)=Q_{\delta}\!\left(\mathrm{Beta}(1+s_m,\,1+c_m)\right).
\end{equation}
so the confidence score in Eq.~\ref{eq:beta-confidence} reflects both empirical correctness and evidence volume. Weakly tested broad items therefore remain cautious, while repeatedly validated broad items are trusted more strongly. Proposition~\ref{prop:posterior-surfacing} gives the resulting posterior error-budget guarantee, and Appendix~\ref{app:beta-vs-hoeffding} motivates the Beta choice over variance-oblivious alternatives.

At inference time, the system retrieves small candidate sets from broad and local memory by semantic similarity, filters only the retrieved broad items using label-sensitive confidence thresholds, serializes the surviving broad items together with retrieved local rules into a structured guardrail prompt, and asks the inference model to output the final decision. If no broad item survives filtering and no local rule is retrieved, the system falls back to the base guardrail $\baseg(x_t)$.

We use separate thresholds $\thref$ and $\thallow$ for refusal-oriented and allow-oriented broad memory. 
A retrieved broad item $m$ is surfaced only if
\begin{equation}
\label{eq:gating-rule}
\conf(m) \ge \tau(\ell_m),
\qquad
\tau(\ell_m)=
\begin{cases}
\thref, & \ell_m=1,\\
\thallow, & \ell_m=0.
\end{cases}
\end{equation}
This gating rule in Eq.~\ref{eq:gating-rule} provides a practical operating knob for broad-policy reuse: higher thresholds make the system more conservative, while asymmetric thresholds allow safety- or utility-prioritized deployment without changing the base guardrail.

\paragraph{Why conservative confidence rather than empirical accuracy alone.}
In sparse-feedback regimes, empirical accuracy can overstate the reliability of weakly tested broad items: a policy validated once and a policy validated many times may both appear perfect. Using a posterior lower bound avoids surfacing such brittle broad memory too early, while still allowing repeatedly validated broad items to influence inference more strongly.

\subsection{Formal design rationale}
\label{sec:formal-support}

The preceding components are motivated by two standard observations about sparse adaptive decision making. We include them here only to clarify where \textsc{LiSA} allocates refinement and when it reuses memory; Appendix~\ref{app:theory-proofs} gives the corresponding formal statements.

\paragraph{Refine broad states with label conflict.}
Broad policy memory is useful because it lets sparse reports generalize beyond individual cases, but its main failure mode is collapsing nearby cases with different labels into the same reuse state. For a broad state $z$, let $\eta_B(z)=\Pr(Y=1\mid Z=z)$ denote the REFUSE rate within $z$. If one broad decision covers all cases in $z$, refinement can only recover errors on the minority label, whose total mass is
\begin{equation}
\label{eq:conflict-mass}
\Pr(Z=z)\min\{\eta_B(z),1-\eta_B(z)\}.
\end{equation}
This product is zero for label-pure states and largest when both labels have substantial support. This motivates using local policies selectively in mixed-label regions, rather than refining every broad abstraction.

\paragraph{Gate reuse by evidence, not empirical accuracy alone.}
Sparse feedback also makes newly induced broad memory look more reliable than it is. A broad policy item with one support and no contradiction has empirical accuracy $1.0$, but little evidence. \textsc{LiSA} therefore scores each broad item $m$ with support and contradiction counts $(s_m,c_m)$ using the lower posterior quantile
\[
\mathrm{Conf}(m)=Q_\delta(\mathrm{Beta}(1+s_m,1+c_m)).
\]
The reuse rule $\mathrm{Conf}(m)\ge \tau(\ell_m)$ keeps weakly tested broad items from influencing inference too early, while allowing repeatedly validated broad items to be surfaced. This is not an end-to-end correctness guarantee for the prompted guardrail, but an evidence-sensitive criterion for controlling broad-memory reuse under sparse and noisy feedback.

\subsection{Lifelong online--offline adaptation}
\label{subsec:lifelong-loop}

The system alternates between online deployment and offline memory refresh. Online, current memory guards new inputs; offline, accumulated reports are folded back into memory by rebuilding broad items, regenerating local items in mixed-label regions, and updating broad-memory confidence statistics. The refreshed memory is redeployed in the next round, enabling continual adaptation without repeated fine-tuning. Algorithm~\ref{alg:lifelong-adaptation} summarizes the \textsc{LiSA} online--offline adaptation procedure.

\paragraph{Why global refresh rather than append-only growth.}
New reports do not merely add rules; they can reveal that existing items are redundant, overly broad, or near a previously unseen mixed-label boundary. Append-only updates would therefore accumulate overlapping abstractions and make memory increasingly order-dependent. We instead rebuild the policy set from the cumulative report bank so that broad items can be re-merged, conflict-heavy regions re-identified, and local refinements regenerated under the full evidence. In practice, only newly reported failures incur LLM-based policy induction; existing items are re-clustered and merged deterministically over their stored statements and statistics, so the refresh cost scales with new reports rather than the size of accumulated memory.

\paragraph{Preserving runtime evidence across refresh.}
A key challenge in rebuilding memory is handling online support and contradiction counts. Discarding them wastes deployment evidence, but transferring them across semantically similar yet distinct abstractions is unreliable. We therefore carry over runtime statistics only for broad policy statements that survive the rebuild, keeping confidence estimates meaningful without propagating evidence beyond the items that originally collected it.

% Algorithm~\ref{alg:lifelong-adaptation} summarizes the full \textsc{LiSA} online--offline adaptation procedure.

\begin{algorithm}[t]
\caption{\textsc{LiSA}: Lifelong Safety Adaptation}
\label{alg:lifelong-adaptation}
\begin{algorithmic}[1]
\Require Base guardrail $\baseg$, Memory $\reportbank$, Broad policies $\broadmem$, Local policies $\localmem$
\State Initialize $\reportbank, \broadmem, \localmem \gets \emptyset$
\For{each deployment round}
    \Statex \quad\textit{\textbf{// Online phase: guard each input with current memory}}
    \For{each input $x_t$ in the round}
        \State $\broadmem_{Ret} \gets \text{Retrieve}(x_t, \broadmem),\ \ \localmem_{Ret} \gets \text{Retrieve}(x_t, \localmem)$
            \Comment{retrieve broad / local policies}
        \State $\broadmem_{Ret} \gets \{\, m \in \broadmem_{Ret} : \conf(m) \ge \tau(\ell_m)\,\}$
            \Comment{confidence gating}
        \If{$\broadmem_{Ret} \cup \localmem_{Ret} = \emptyset$}
            \State $\hat y_t \gets \baseg(x_t)$
                \Comment{if no policy applies $\to$ fall back to base decision}
        \Else
            \State $\hat y_t \gets \baseg(x_t,\, \broadmem_{Ret} \cup \localmem_{Ret})$
                \Comment{decision with retrieved policies}
        \EndIf
        \State Append correction $(x_t, \tilde y_t)$ to $\reportbank$ if any
            \Comment{collect new sparse feedback}
    \EndFor
    \Statex \quad\textit{\textbf{// Offline phase: refresh memory from accumulated reports}}
    \State $\broadmem \gets \broadmem \cup \text{InduceBroad}(\reportbank)$
        \Comment{abstract new failures into broad policy candidates}
    \State $\broadmem \gets \text{Cluster}(\broadmem)$
        \Comment{merge similar items; carry over evidence $(s_m,\, c_m)$}
    \State $\localmem \gets \text{InduceLocal}(\reportbank)$
        \Comment{regenerate boundary rules in mixed-label regions}
\EndFor
\end{algorithmic}
\end{algorithm}

\section{Experimental setting}
\label{sec:experiments}

We evaluate lifelong guardrail adaptation from sparse user-reported failures: the base guardrail remains fixed, feedback is provided only for misclassified inputs, and memory is refreshed periodically rather than through repeated fine-tuning. This setup allows us to investigate whether structured policy memory can improve a lightweight guardrail over time, and whether the conservative mechanisms from Section~\ref{sec:formal-support}---local refinement in mixed-label regions and evidence-gated broad reuse---behave as intended in practice. We organize the experiments around four questions:
\begin{description}[leftmargin=1.5em]
    \item[\textbf{RQ1 (Sparse adaptation).}] Can a fixed base guardrail improve over deployment rounds using only sparse failure reports?

    \item[\textbf{RQ2 (Component roles).}] Do local refinement and confidence-gated reuse contribute in the distinct ways suggested by the design rationale?

    \item[\textbf{RQ3 (Robustness).}] Is adaptation stable when reported labels are noisy, and is this stability specifically tied to evidence-aware confidence measurement?

    \item[\textbf{RQ4 (Cost vs.\ scaling).}] Does structured memory provide a better cost--performance trade-off than simply using a larger un-adapted guardrail?
\end{description}
Sections~\ref{sec:main-results}--\ref{sec:cost-performance} address these questions in order.

\subsection{Datasets}

We evaluate on three binary guardrailing benchmarks with complementary risk profiles: \textbf{PrivacyLens+}~\citep{privacylens}, \textbf{ConFaide+}~\citep{confaide}, and \textbf{AgentHarm}~\citep{agentharm}. The ``+'' variants of PrivacyLens and ConFaide include ambiguous contextual cases introduced by~\citet{privacyreasoning}. These privacy benchmarks are useful for evaluating local decision boundaries because labels can change under subtle contextual differences. AgentHarm broadens the evaluation beyond privacy to harmful agent behavior. We map all datasets into a shared binary decision space, \textsc{allow} and \textsc{refuse}. When multiple examples are derived from the same base scenario, we keep them in the same split to avoid overlap between adaptation and held-out evaluation.

\paragraph{Deployment simulation.}

We simulate deployment over $N$ days. On each day, the guardrail predicts decisions for a stream of cases. At the end of the day, each adaptive method receives feedback only on the cases it misclassified, paired with the reported label, and updates its memory before the next day. Held-out evaluation is performed after each update on a fixed test set that is never used for adaptation. For noisy-feedback experiments, each reported failure label is independently flipped with probability $\rho$ (i.e., noise ratio) before being passed to the adaptive method.

\subsection{Baselines}

We compare \textsc{LiSA} against four baselines. \textbf{Pure Prediction} uses the base guardrail without adaptation. \textbf{AGrail}~\citep{agrail} updates a test-time checklist from observed failures and a previous history. \textbf{Synapse}~\citep{synapse} stores corrected cases and retrieves similar past examples at inference time. \textbf{ReasoningBank}~\citep{reasoningbank} converts failures into reusable natural-language memories and retrieves them as policy-like guidance. In this setting, ReasoningBank serves as a broad-policy-only memory baseline: it tests whether policy abstraction alone is sufficient, without \textsc{LiSA}'s conflict-aware local refinement or confidence-gated surfacing. All adaptive methods receive only their own day-end failure reports and never observe dense labels for the full stream. Synapse and ReasoningBank were originally designed for reasoning tasks; Appendix~\ref{app:baseline-prompting} describes how we adapt them to this guardrail setting.

\paragraph{Implementation details}

We instantiate the online guardrail with two lightweight models, 
\texttt{Gemini} \texttt{-3.1-flash-lite} and \texttt{Claude-Haiku-4.5}, to test whether the adaptation pattern depends on a particular model family. For both ReasoningBank and \textsc{LiSA}, the offline manager for policy induction  is \texttt{Gemini-3.1-pro}, and semantic retrieval uses \texttt{Gemini-embedding-001}~\citep{geminiembedding}. These offline components are kept fixed across online guardrail configurations. Adaptive methods use fixed prompt templates across datasets. We report accuracy and macro-F1 on the same held-out split across methods, averaged over five seeds. Additional split sizes, retrieval limits, thresholds, prompts, and noise construction details are provided in Appendix~\ref{app:details}.

\section{Empirical results}

\subsection{Main results: lifelong adaptation from sparse failure reports}
\label{sec:main-results}

\begin{figure*}[t]
    \centering
    \includegraphics[trim=0 0 0 10, clip,width=\linewidth]{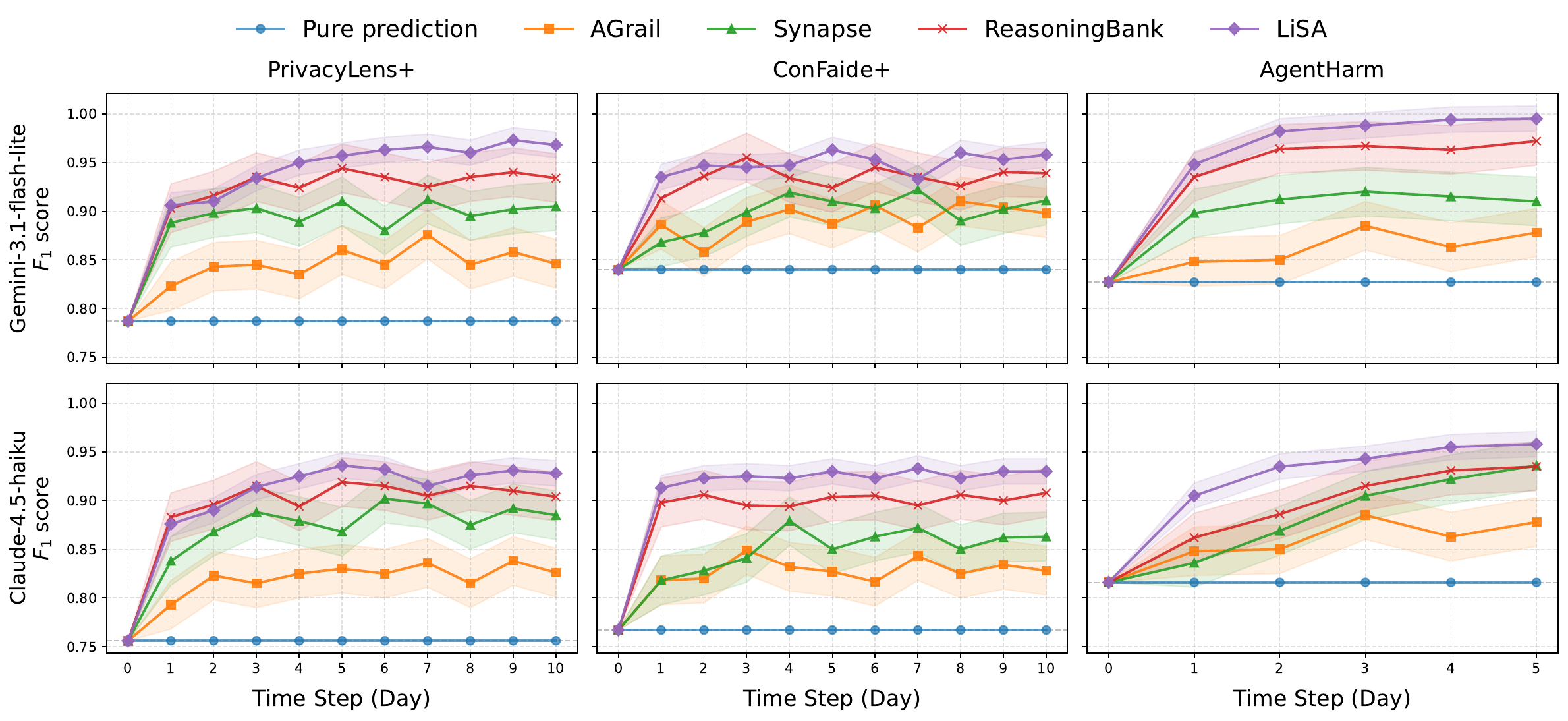}
    \caption{\textbf{Lifelong safety adaptation results}. Held-out F1 versus deployment day with sparse failure reports. Each curve is averaged over five seeds, and shaded regions indicate standard deviation. Upper row: \texttt{Gemini-3.1-flash-lite} as the online guardrail; lower row: \texttt{Claude-Haiku-4.5}.}
    \label{fig:main-curves}
    \vspace{-4mm}
\end{figure*}

Figure~\ref{fig:main-curves} shows held-out macro-F1 over deployment days. The fixed \textbf{Pure Prediction} baseline serves as the un-adapted reference performance. Memory-based methods improve over this baseline, but the type of memory matters. \textbf{AGrail} and \textbf{Synapse} yield limited gains, suggesting that checklist-style summaries or isolated case reuse do not provide enough transferable structure under sparse feedback. \textbf{ReasoningBank} performs better, indicating that broad policy abstraction is a useful unit of lifelong adaptation.

\textbf{\textsc{LiSA}} achieves the strongest performance across all three benchmarks and under both online guardrail models. The gain over ReasoningBank is moderate but consistent, which is the expected pattern if broad abstraction already captures much of the reusable signal, while \textsc{LiSA}'s additional mechanisms mainly address the failure modes of broad reuse. This matches the design rationale in Section~\ref{sec:formal-support}: conflict-aware local rules help in regions where semantically similar cases have different labels, and confidence-gated surfacing limits the influence of broad policies that have not yet accumulated enough support. The next subsection isolates each component empirically.

\subsection{Component roles: local refinement improves boundaries, gating stabilizes reuse}
\label{sec:component-roles}

\begin{wraptable}{r}{0.44\linewidth}
\vspace{-12pt}
\centering
\small
\caption{\textbf{Ablation study of \textsc{LiSA}.} Final-day macro-F1 and standard deviation across five seeds, averaged across benchmarks with \texttt{Gemini-3.1-flash-lite}.}
\label{tab:component-ablation}
\vspace{1mm}
\setlength{\tabcolsep}{4pt}
\renewcommand{\arraystretch}{1.12}
\begin{tabular}{lc}
\toprule
\textbf{Ablation} & \textbf{F1} \\
\midrule
\textbf{\textsc{LiSA}}   & 0.962 ($\pm$ 0.013) \\
w/o Conf gate   & 0.959 ($\pm$ 0.025) \\
w/o Local rules & 0.931 ($\pm$ 0.021) \\
w/o both        & 0.925 ($\pm$ 0.038)  \\
\bottomrule
\end{tabular}
\vspace{-10pt}
\end{wraptable}

Table~\ref{tab:component-ablation} separates the two additions \textsc{LiSA} makes on top of broad policy memory, in the noise-free regime ($\rho=0\%$). The largest mean drop comes from removing local rules. This is consistent with the conflict-mass theoretical prediction in Eq.~\ref{eq:conflict-mass}: when nearby cases split labels, a single broad policy must cover both sides of a local boundary, and narrower rules provide the additional resolution needed in these regions. Removing confidence gating, by contrast, has only a small effect on the mean but noticeably increases seed variance; the same pattern becomes stronger when both mechanisms are removed. 

In conclusion, local refinement drives the boundary-level F1 gain, and confidence gating stabilizes this improvement. The latter effect becomes more pronounced once reported labels are noisy, as we show in Section~\ref{sec:noise-robustness}.

\subsection{Robustness under noisy user reports and its source}
\label{sec:noise-robustness}

Table~\ref{tab:noise-robustness} evaluates adaptation when reported failure labels are corrupted before memory refresh. The results reveal a trade-off between transfer and noise sensitivity. Broad abstraction transfers well under clean feedback: \textbf{ReasoningBank} is the strongest baseline at $\rho=0\%$. Under noise, however, a mislabeled report can become reusable guidance and affect many later decisions, causing performance to drop sharply. Direct case retrieval, as in \textbf{Synapse}, localizes such errors and degrades more gradually, but transfers less effectively and has lower clean-feedback performance.

\begin{table}[t]
\centering
\begin{minipage}[t]{0.5\linewidth}
\centering
\small
\caption{\textbf{Noise robustness.} Final-day F1 score averaged across benchmarks (five seeds) with \texttt{Gemini-3.1-flash-lite}. \textit{$\rho$: label-flip ratio.}}
\label{tab:noise-robustness}
\vspace{3mm}
\setlength{\tabcolsep}{5pt}
\renewcommand{\arraystretch}{1.15}
\begin{tabular}{lccc}
\toprule
\textbf{Method} & $\rho{=}0\%$ & $\rho{=}10\%$ & $\rho{=}20\%$ \\
\midrule
Pure Prediction & 0.818 & 0.818 & 0.818 \\
AGrail          & 0.857 & 0.846 & 0.832 \\
Synapse         & 0.891 & 0.878 & 0.863 \\
ReasoningBank   & 0.936 & 0.882 & 0.857 \\
\midrule
\textbf{\textsc{LiSA}}   & 0.962 & 0.933 & 0.917 \\
\bottomrule
\end{tabular}
\end{minipage}
\hfill
\begin{minipage}[t]{0.46\linewidth}
\centering
\small
\caption{\textbf{Confidence measurement ablation.} Final-day macro-F1 at $\rho{=}20\%$, same setting as left. All variants share \textsc{LiSA}'s two-level memory and differ only in how retrieved broad items are gated.}
\label{tab:beta-ablation}
\vspace{3mm}
\setlength{\tabcolsep}{5pt}
\renewcommand{\arraystretch}{1.15}
\begin{tabular}{lc}
\toprule
\textbf{Confidence} & $\rho{=}20\%$ \\
\midrule
No gating              & 0.854 ($\pm$ 0.044) \\
Accuracy $s/(s{+}c)$   & 0.883 ($\pm$ 0.032) \\
\midrule
\textbf{Beta Quantile} & 0.917 ($\pm$ 0.025) \\
\bottomrule
\end{tabular}
\end{minipage}
\vspace{-2mm}
\end{table}

\textbf{\textsc{LiSA}} retains most of its clean-feedback gain under both noise levels. This indicates that it preserves the transfer benefits of broad abstraction without letting weak policies dominate inference too early. \textsc{LiSA} does not identify noisy reports directly; rather, confidence-gated reuse requires broad policies to accumulate evidence before surfacing, while contradictions lower their posterior confidence. This provides an evidence-sensitive reuse mechanism that is more stable than unconditional retrieval, consistent with the posterior surfacing rule in Section~\ref{subsec:gating}.

\paragraph{Confidence measurement as the stabilizing factor.}
\label{sec:beta-ablation}
Table~\ref{tab:beta-ablation} isolates the confidence mechanism at $\rho=20\%$. All variants use the same broad and local memory; they differ only in how retrieved broad items are surfaced. Without gating, performance drops substantially, showing that local refinement alone cannot prevent noisy broad policies from being reused without sufficient evidence. Empirical-accuracy gating helps, but it ignores evidence volume: a policy with one support and no contradiction and a policy with many supports and no contradictions both receive accuracy $1.0$. The Beta lower quantile separates these cases, blocking weakly tested broad policies early while allowing repeatedly validated policies to be reused. This makes it the most stable reuse rule among the tested variants.

\subsection{Latency--performance trade-off}
\label{sec:cost-performance}

\begin{wrapfigure}{r}{0.37\linewidth}
\vspace{-14pt}
\centering
\includegraphics[width=\linewidth]{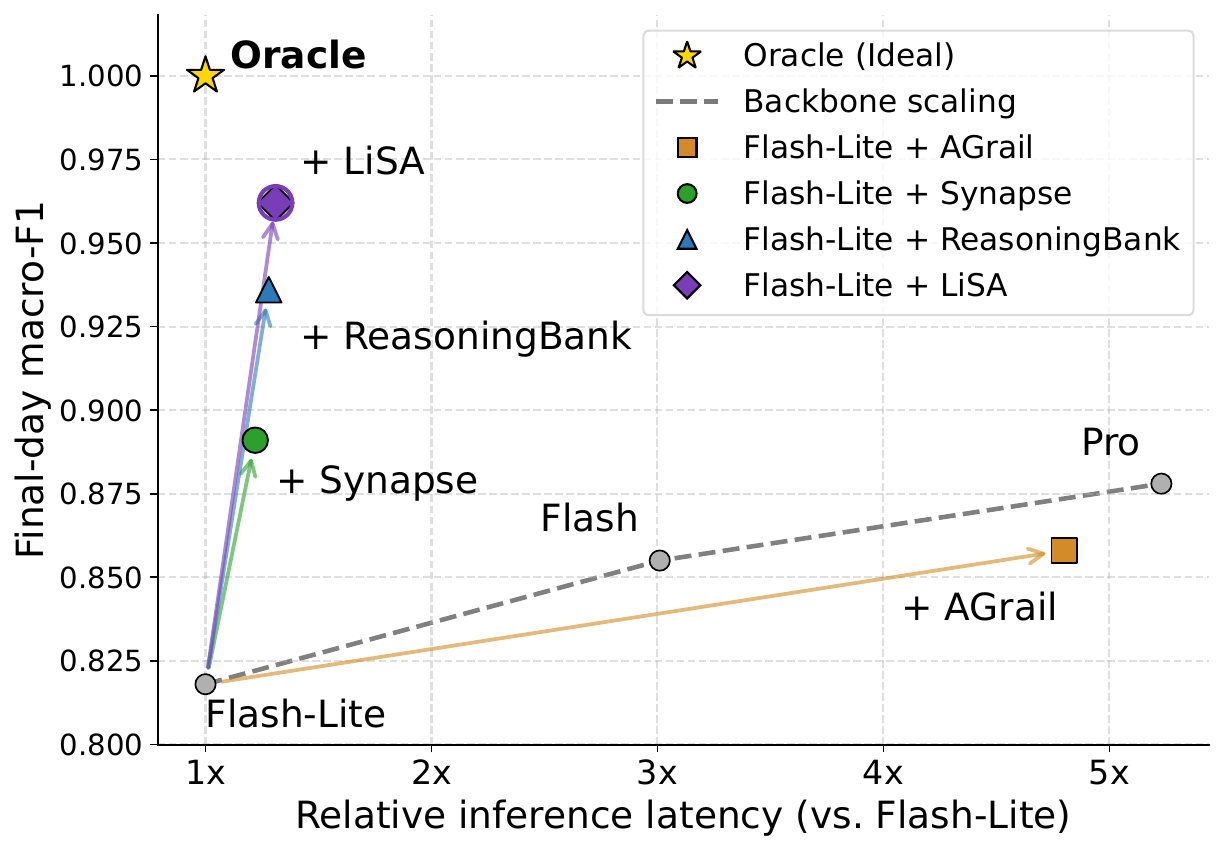}
\vspace{-12pt}
\caption{\textbf{Latency--F1 trade-off.} Memory-based adaptation pushes the frontier beyond base-model scaling, with \textsc{LiSA} closest to the oracle.}
\label{fig:cost-pareto}
\vspace{-11pt}
\end{wrapfigure}

Finally, we compare scaling a static guardrail with adapting a smaller one. Figure~\ref{fig:cost-pareto} plots final-day macro-F1 against average inference latency relative to \texttt{Gemini-3.1-flash-lite}. Larger un-adapted backbones trace the usual scaling frontier: higher F1 at higher per-query latency. Adaptive methods behave differently. \textbf{AGrail} remains below this frontier, as checklist reasoning adds latency without commensurate gains under sparse feedback. In contrast, memory-based reuse shifts the frontier upward: \textbf{Synapse} benefits from direct case reuse, \textbf{ReasoningBank} from policy abstraction, and \textbf{\textsc{LiSA}} moves closest to the oracle by adding conflict-aware local refinement and confidence-gated reuse.

Memory-based adaptation has a different cost structure from backbone scaling. 
Offline refresh is triggered only by sparse reported failures, runs outside the per-query serving path, and amortizes over many later decisions once the induced policies are reused. Detailed offline token costs are reported in Appendix~\ref{app:offline-cost}. Thus, in sparse-feedback guardrail deployment, structured memory offers a more efficient path to stronger decisions than scaling a backbone guardrail alone.

\section{Conclusion}

We studied lifelong safety adaptation under sparse, noisy user-reported failures. Our results support a simple claim: effective adaptation does not require fine-tuning or larger models, but does require calibrated reuse of deployment experiences. \textsc{LiSA} realizes this thesis through structured policy memory: abstractions generalize sparse failures, local rules preserve boundary cues in mixed-label regions, and confidence gating prevents over-generalization of weakly supported memory. Across benchmarks, this conservative memory-driven design improves guardrail performance, remains robust to noisy feedback, and pushes the latency--performance frontier beyond base-model scaling.

More broadly, \textsc{LiSA} suggests a direction for future guardrails: they should adapt to deployment experience, but only through evidence-calibrated reuse. Static guardrails cannot anticipate the long tail of local norms and evolving user expectations; yet unconstrained adaptation can introduce over-refusal or unsafe permission. Conservative policy induction offers a practical middle ground, allowing guardrails to learn from their operating environments while preserving evidence-grounded caution.

% \section*{Acknowledgements}

\bibliography{main}

\newpage

\appendix

\section{Limitations and practical implications}
\label{sec:limitations}

We close by clarifying the scope of our empirical evidence and the practical implications for deploying \textsc{LiSA} beyond the controlled benchmark setting.

\paragraph{Benchmark simulation as a controlled deployment proxy.}
Our evaluation uses benchmark-based deployment simulations rather than logs from a live deployed guardrail. This design enables controlled, reproducible comparisons while preserving key deployment constraints: sparse feedback, corrections only for each method's own mistakes, periodic memory refresh rather than repeated fine-tuning, and held-out evaluation never used for adaptation. Still, simulations cannot capture all properties of live deployments, where reports may be delayed, correlated, systematically noisy, or shaped by evolving organizational norms and user expectations. Our benchmarks are also English-language and focused on privacy- and safety-sensitive decisions, leaving multilingual~\citep{multilingual}, culturally heterogeneous~\citep{culturalsafety}, and domain-specific settings to future work. 
We therefore view the experiments as controlled evidence for \textsc{LiSA}'s adaptation mechanism under deployment-like constraints, rather than as a claim that \textsc{LiSA} is the universally optimal strategy. We encourage practitioners to tailor \textsc{LiSA} to their own deployment settings by adapting our components to local operational constraints.

\paragraph{Threshold calibration as practical guidance.}
The default threshold $\thref=\thallow=0.55$ should be understood in this spirit. Although the value is only slightly above chance, \textsc{LiSA} applies it to the lower $5\%$ posterior quantile of $\mathrm{Beta}(1+s_m,1+c_m)$ for broad policy items, so it imposes a nontrivial evidence requirement. A broad policy with no contradictions is not surfaced after one, two, three, or four supports, and first passes the threshold after about five contradiction-free supports. If contradictions are observed, more evidence is required; for example, roughly $(s_m,c_m)=(7,1)$, $(9,2)$, $(11,3)$, or $(15,5)$ is needed to pass. Thus, the symmetric threshold $\thref=\thallow=0.55$ blocks one-off broad memories and weakly supported broad policies while still allowing adaptation once modest evidence accumulates. We recommend it only as a conservative starting point, not as a deployment-independent default: the threshold should be calibrated to the target application, the expected feedback quality, and the relative cost of false accepts versus false refusals.

\paragraph{Asymmetric thresholds as a deployment interface.}
While \textsc{LiSA} supports asymmetric thresholds for refusal-oriented and allow-oriented broad memory, and Proposition~\ref{prop:label-budget} formalizes the resulting separate posterior error budgets, we do not claim to demonstrate calibrated control over the false-accept/false-refuse trade-off in this work. In our benchmark simulations, surviving broad policy items often became high-confidence after enough evidence accumulated, making asymmetric threshold sweeps less diagnostic over the 5--10 day horizon. We therefore retain asymmetric thresholds as a practical deployment interface rather than a quantitatively validated control: real-world environments are likely to exhibit more heterogeneous confidence profiles, noisier feedback, and application-specific costs for over-acceptance and over-refusal, where allocating separate posterior error budgets to refusal-oriented and allow-oriented broad memory may be useful.

\paragraph{Scope of formal results.}
Our formal results should also be read as component-level design support rather than end-to-end guarantees. Proposition~\ref{prop:conflict-mass} motivates allocating local refinement to mixed-label regions, while Proposition~\ref{prop:posterior-surfacing} justifies confidence-gated broad reuse as an evidence-sensitive surfacing rule. Neither result models the full prompted guardrail, retrieval dynamics, or the closed online--offline feedback loop. The empirical gains in Sections~\ref{sec:main-results}--\ref{sec:cost-performance} therefore reflect the joint behavior of \textsc{LiSA}'s memory mechanism and the underlying inference model, while the propositions clarify why the individual memory operations are reasonable.

\section{Formal results and proofs}
\label{app:theory-proofs}

\subsection{Notation}
\label{app:setup}

Let $\phi_B:\mathcal{X}\to\mathcal{Z}$ denote the coarse reuse state induced by broad memory and $\phi_{BL}:\mathcal{X}\to\mathcal{U}$ the refined reuse state induced after adding conflict-aware local splits, so there exists a measurable map $T$ with $\phi_B=T\circ\phi_{BL}$. Write
\[
Z=\phi_B(X), \qquad U=\phi_{BL}(X),
\]
\[
\eta_B(z)=\Pr(Y=1\mid Z=z), \qquad \eta_{BL}(u)=\Pr(Y=1\mid U=u).
\]
Let $\gfun(p)=\min\{p,1-p\}$, so that
\[
R^\star_{\phi_B}=\mathbb{E}[\gfun(\eta_B(Z))], \qquad
R^\star_{\phi_{BL}}=\mathbb{E}[\gfun(\eta_{BL}(U))].
\]
For a broad state $z$, $\Delta(z)$ denotes the reduction in Bayes $0$--$1$ risk attainable by any measurable refinement supported on $\{Z=z\}$.

\subsection{Conflict mass bound on refinement gain}
\label{app:conflict-mass-proof}

\begin{proposition}[Conflict mass bounds attainable refinement gain]
\label{prop:conflict-mass}
For every broad state $z$,
\[
\Delta(z) \;\le\; \Pr(Z=z)\cdot\min\{\eta_B(z),\,1-\eta_B(z)\},
\]
with $\Delta(z)=0$ on pure states. Under a unit-cost state-level refinement model with independent budgets, ranking broad states by the conflict mass
\[
\Pr(Z=z)\min\{\eta_B(z),1-\eta_B(z)\}
\]
maximizes the attainable upper bound on total risk reduction.
\end{proposition}

\begin{proof}
Fix a broad state $z$. For any refinement of $\{Z=z\}$ into sub-states,
\[
\mathbb{E}[\gfun(\eta_{BL}(U))\mid Z=z]\;\ge\; 0,
\]
so
\[
\Delta(z)
\;\le\; \gfun(\eta_B(z))\cdot\Pr(Z=z)
\;=\; \Pr(Z=z)\cdot\min\{\eta_B(z),\,1-\eta_B(z)\},
\]
which proves the first claim. In particular, $\Delta(z)=0$ when $\eta_B(z)\in\{0,1\}$.

For the ranking statement, consider a refinement model in which (i) each broad state can be refined independently at unit cost, and (ii) refinement of one state does not affect attainable gain on any other state. Under these assumptions, the attainable total risk reduction over a budget $B$ is bounded above by
\[
\sum_{z\in S}\Pr(Z=z)\min\{\eta_B(z),1-\eta_B(z)\},
\]
where $S$ is the set of refined states. Selecting the top $B$ states by conflict mass maximizes this upper bound, since the problem reduces to selecting the top $B$ nonnegative summands.
\end{proof}

The optimality statement is with respect to the attainable upper bound under the stated model. It should be interpreted as a ranking criterion rather than an absolute guarantee: cost structures or feasibility constraints that couple refinements across states may alter the optimal allocation.

\subsection{Corollary: standard refinement inequality}
\label{app:refinement-corollary}

Proposition~\ref{prop:conflict-mass} recovers the classical refinement inequality as a corollary. Since $\phi_B=T\circ\phi_{BL}$, we have $\eta_B(Z)=\mathbb{E}[\eta_{BL}(U)\mid Z]$. Concavity of $\gfun$ and Jensen's inequality give
\[
\gfun(\eta_B(Z))\;\ge\;\mathbb{E}[\gfun(\eta_{BL}(U))\mid Z],
\]
and taking expectations yields
\[
R^\star_{\phi_{BL}}\le R^\star_{\phi_B}.
\]
The inequality is strict on any positive-measure broad state that the refinement splits across the Bayes boundary $1/2$; these are precisely the states with strictly positive conflict mass. Proposition~\ref{prop:conflict-mass} therefore localizes this classical fact by attributing attainable gain to conflict mass at the state level.

\subsection{Posterior surfacing guarantee}
\label{app:posterior-surfacing-proof}

\begin{proposition}[Posterior surfacing guarantee]
\label{prop:posterior-surfacing}
Let $q$ be a broad policy item and
\[
\conf(q)=Q_\delta(\mathrm{Beta}(1+s_q,1+c_q)).
\]
Under the gating rule $\conf(q)\ge\tau$, every surfaced broad item satisfies
\[
\Pr(\theta_q<\tau\mid s_q,c_q) \;\le\; \delta,
\qquad
\mathbb{E}[1-\theta_q \mid s_q,c_q,\, q\text{ surfaced}] \;\le\; (1-\tau)+\delta.
\]
\end{proposition}

\begin{proof}
Let $L_\delta(s,c)$ denote the lower $\delta$-quantile of $\mathrm{Beta}(1+s,1+c)$, so $\conf(q)=L_\delta(s_q,c_q)$. By the definition of the lower quantile,
\[
\Pr(\theta_q < L_\delta(s_q,c_q)\mid s_q,c_q) \;\le\; \delta.
\]
Whenever $q$ is surfaced, $L_\delta(s_q,c_q)\ge\tau$, hence
\[
\Pr(\theta_q<\tau\mid s_q,c_q)
\;\le\;\Pr(\theta_q<L_\delta(s_q,c_q)\mid s_q,c_q)
\;\le\;\delta.
\]

For the expected-error bound, surfacing is $(s_q,c_q)$-measurable, so
\[
\mathbb{E}[1-\theta_q\mid s_q,c_q,\,q\text{ surfaced}]
\;=\;\mathbb{E}[1-\theta_q\mid s_q,c_q].
\]
Splitting on $\{\theta_q\ge\tau\}$ and $\{\theta_q<\tau\}$ and using $1-\theta_q\le 1-\tau$ on the first event and $1-\theta_q\le 1$ on the second,
\[
\mathbb{E}[1-\theta_q\mid s_q,c_q]
\;\le\; (1-\tau)\cdot\Pr(\theta_q\ge\tau\mid s_q,c_q)
      + \Pr(\theta_q<\tau\mid s_q,c_q)
\;\le\; (1-\tau)+\delta. \qedhere
\]
\end{proof}

\subsection{Corollary: label-sensitive thresholds as separate error budgets}
\label{app:label-sensitive-budget}

\begin{proposition}[Error-budget interpretation]
\label{prop:label-budget}
Under the gating rule of Section~\ref{subsec:gating}, every surfaced broad refusal-oriented item ($\ell_q=1$) satisfies
\[
\Pr(\theta_q<\thref\mid s_q,c_q)\le\delta,
\qquad
\mathbb{E}[1-\theta_q\mid s_q,c_q,\,q\text{ surfaced}]\le (1-\thref)+\delta,
\]
and every surfaced broad allow-oriented item ($\ell_q=0$) satisfies the same bounds with $\thref$ replaced by $\thallow$. Setting $\thref\ne\thallow$ therefore allocates independent posterior error budgets to refusal-oriented and allow-oriented broad memory.
\end{proposition}

\begin{proof}
Immediate from Proposition~\ref{prop:posterior-surfacing} applied separately within each label class.
\end{proof}

The quantity $1-\thref$ upper bounds the posterior expected error of surfaced broad refusal-oriented memory, and analogously for allow-oriented broad memory. Operators prioritizing avoidance of over-refusal can raise $\thref$, while those prioritizing avoidance of over-acceptance can raise $\thallow$; the two budgets are set independently.

\subsection{Adaptive tightness: Beta vs.\ Hoeffding}
\label{app:beta-vs-hoeffding}

Proposition~\ref{prop:posterior-surfacing} holds for any valid lower credible bound, but the specific choice of the Beta-posterior quantile governs how quickly broad items separate from the threshold in deployment. Write $\hat\theta_n=s/(s+c)$ with $n=s+c$, and let
\[
L^\beta_\delta(s,c)=Q_\delta(\mathrm{Beta}(1+s,1+c)),
\qquad
L^H_\delta(s,c)=\hat\theta_n-\sqrt{\frac{\log(1/\delta)}{2n}}.
\]

A standard quantile approximation gives, for fixed $\delta\in(0,1/2)$ and large $n$,
\[
L^\beta_\delta(s,c)-\hat\theta_n
\;\approx\;
-z_{1-\delta}\sqrt{\frac{\hat\theta_n(1-\hat\theta_n)}{n}},
\qquad
L^H_\delta(s,c)-\hat\theta_n
=
-\sqrt{\frac{\log(1/\delta)}{2n}},
\]
where $z_{1-\delta}$ is the $(1-\delta)$-quantile of the standard normal distribution.

The Beta lower bound therefore scales with the sample variance $\hat\theta_n(1-\hat\theta_n)$, which shrinks as $\hat\theta_n\to 0$ or $\hat\theta_n\to 1$, whereas the Hoeffding lower bound uses the worst-case variance $1/4$ and depends only on $n$. As a consequence, broad items with empirical reliability close to $0$ or $1$, corresponding to clearly unreliable or clearly reliable broad memory, separate from the surfacing threshold faster under the Beta score, while broad items with $\hat\theta_n\approx 1/2$ remain cautious under both. In this sense, the Beta choice adapts conservatism to the empirical reliability of each broad item.

\subsection{Monotonicity and scope}
\label{app:monotonicity-and-scope}

\begin{remark}[Monotonicity]
\label{rem:beta-monotone}
$\conf(q)=L^\beta_\delta(s_q,c_q)$ is nondecreasing in $s_q$ and nonincreasing in $c_q$. Two broad items with the same empirical mean but different total evidence therefore receive different confidence scores, with more evidence producing a tighter lower bound.
\end{remark}

\begin{remark}[Scope]
\label{rem:scope}
Proposition~\ref{prop:conflict-mass} concerns the representation-level Bayes risk attainable under refinement and does not model the prompt-level behavior of the deployed inference model. Proposition~\ref{prop:posterior-surfacing} is a statement about the posterior reliability of a surfaced broad item conditional on its accumulated evidence; it does not capture selection dynamics across the full memory, since post-surfacing evidence depends in part on earlier broad-memory gating decisions. The periodic offline refresh described in Section~\ref{subsec:lifelong-loop} partially mitigates this by re-evaluating broad items as the report bank grows. We therefore interpret these results as supporting each component of the method in isolation, rather than as end-to-end guarantees about the full adaptive system.
\end{remark}

\section{Experiments and implementation details}
\label{app:details}

\subsection{Deployment simulation and splits}

All experiments use the three benchmarks reported in the main text: PrivacyLens+, ConFaide+, and AgentHarm. The ``+'' suffix for PrivacyLens+ and ConFaide+ denotes the expanded versions used in \citet{privacyreasoning}, which augment the original datasets with additional ambiguous-context variants. We map each dataset into the shared binary label space \textsc{allow}/\textsc{refuse}, implemented as \texttt{appropriate}/\texttt{inappropriate}. Splits are group-preserving: when multiple rows are derived from the same base scenario, they are assigned together so that adaptation examples and held-out evaluation examples do not share the same scenario group.

For lifelong adaptation, each method receives a stream of deployment queries and is evaluated repeatedly on a fixed held-out set. On each day, a method first predicts labels for that day's streamed cases. At day end, it receives only the cases it misclassified, paired with the reported label, and updates its memory before the next day. The fixed held-out set is never used for adaptation and is evaluated after each daily update. Unless otherwise stated, results are averaged over five seeds.

\begin{table}[h]
\centering
\small
\vspace{-3mm}
\caption{\textbf{Deployment simulation sizes.} Only misclassified streamed cases are reported to the adaptive method at day end. Held-out examples are fixed across days and are never used for adaptation.}
\label{tab:deployment-sizes}
\vspace{3mm}
\setlength{\tabcolsep}{6pt}
\renewcommand{\arraystretch}{1.12}
\begin{tabular}{lccc}
\toprule
\textbf{Benchmark} & \textbf{Daily stream} & \textbf{Days} & \textbf{Held-out evaluation} \\
\midrule
PrivacyLens+ & 50 queries/day & 10 & 500 examples \\
ConFaide+    & 50 queries/day & 10 & 500 examples \\
AgentHarm    & 60 queries/day & 5  & 116 examples \\
\bottomrule
\end{tabular}
\end{table}

AgentHarm uses a shorter deployment horizon because fewer examples are available after constructing a group-preserving held-out split. For noisy-feedback experiments on AgentHarm, we use a fixed 200-example stream and apply label flips only to reported failures, following the same corruption protocol described below.

\subsection{Model and retrieval configuration}
\label{app:config}
The online guardrail is either \texttt{Gemini-3.1-flash-lite} or \texttt{Claude-Haiku-4.5}. The offline manager for reflection and policy induction is \texttt{Gemini-3.1-pro}, and retrieval uses \texttt{Gemini-embedding-001}. All generation calls use temperature $0$ and deployed with Google Cloud Vertex AI~\footnote{\url{https://cloud.google.com/vertex-ai}}. At inference time, retrieved context is bounded: methods retrieve at most five similar cases and two policy-like memory items per memory type. For \textsc{LiSA}, the final prompt contains at most two retrieved local rules and two confidence-filtered broad policy items. All model inference, offline policy-induction, and embedding calls were served through managed Google Cloud Vertex AI APIs; we did not run local model inference or allocate local GPU workers, so hardware details such as worker type and memory were abstracted by the managed API service.

\subsection{Policy gating and noise construction}

\textsc{LiSA} applies the Beta lower-credible-bound confidence score described in Section~\ref{subsec:gating} to broad policy items only, with prior $\mathrm{Beta}(1,1)$ and $\delta=0.05$. Unless otherwise stated, both label-specific thresholds are set to $\thref=\thallow=0.55$. If no local rule is retrieved and no retrieved broad policy item passes the threshold, \textsc{LiSA} does not prompt with empty memory; it falls back to the base guardrail.

For noisy-feedback experiments, only reported failure labels are corrupted. Each reported label is independently flipped with probability $\rho$, and the corrupted label is then used by the adaptation method. The held-out labels used for evaluation are never corrupted.

\subsection{\textsc{LiSA} offline refresh}
\label{app:offline-refresh}

At the end of each deployment day, \textsc{LiSA} refreshes memory from the newly reported failures and the accumulated case history. The refresh has two memory channels. First, newly reported failures are passed to Template 2 to induce broad structured preventive items. These items are stored as general policy candidates with a recommended label and initial support/contradiction counts from the reported batch. Second, \textsc{LiSA} rebuilds conflict-aware local rules from the accumulated case memory. This step clusters semantically similar cases, keeps mixed-label neighborhoods, and creates narrow label-specific rules for the local boundary.

The broad-policy induction step uses the following deterministic grouping and merging procedure around the LLM synthesis call. All failures newly reported at the end of the current deployment day form one induction group; in our experiments this group contains only the model's misclassified streamed cases for that day. Template 2 is then called once per non-empty group with temperature $0$ and is allowed to return at most three structured preventive items. Each item is parsed into \texttt{Title}, \texttt{Description}, \texttt{Content}, \texttt{Recommended label}, and \texttt{Rule type}; invalid or missing labels are resolved to the majority corrected label in the inducing group. The item's provenance is the full set of report ids in the group, its label-skew metadata is the corrected-label histogram of the group, its initial support count is the number of inducing reports whose corrected label matches the item's recommended label, and its initial contradiction count is the remaining number of reports in that group.

Across refreshes, broad items are not appended as independent rules forever. We keep the raw induced candidates and rebuild the broad memory by embedding each canonical structured statement with \texttt{Gemini-embedding-001}, then clustering statements with agglomerative clustering using cosine distance, average linkage, no fixed number of clusters, and distance threshold $0.20$. For each cluster, the representative statement is the member whose embedding has maximum cosine similarity to the cluster centroid. Cluster metadata is obtained by summing support counts, contradiction counts, label-skew histograms, and representative report ids over members. \textsc{LiSA} does not use a second LLM call to rewrite merged clusters; the representative statement is reused verbatim. This makes the cross-day grouping criterion exactly the embedding-cluster membership of induced statements, rather than an unreported manual taxonomy.

Mixed-label regions are detected separately from the broad-policy clusters, using the accumulated case memory rather than LLM-generated policy text. Cases are embedded from their canonical scenario summaries and clustered within each domain namespace with agglomerative clustering using cosine distance, average linkage, no fixed number of clusters, distance threshold $0.20$, and minimum cluster size $2$. A cluster is marked as mixed-label if its corrected labels contain both \textsc{allow}/\texttt{appropriate} and \textsc{refuse}/\texttt{inappropriate}; we do not impose an additional balance threshold beyond at least one case of each label, and record the conflict score as $1-\max_y n_y/\sum_y n_y$. Pure clusters are discarded. 

The local-rule text is rendered deterministically from the mixed cluster. We form a region summary by linearizing the case metadata already present in the input records, and identify decisive pivots as the attributes or textual facets whose common values differ across the \textsc{allow}/\texttt{appropriate} and \textsc{refuse}/\texttt{inappropriate} members of the same cluster. Thus, local memory is generated only from semantic neighborhoods that contain both labels; the input fields only determine how the already-detected boundary is verbalized, rather than serving as hand-written benchmark-specific decision rules.

Existing broad policies are not assigned a separate LLM update prompt. If a broad policy was surfaced during inference, its evidence is updated directly: a label match increments support and a mismatch increments contradiction. During refresh, all raw broad policy candidates are re-clustered by embedding similarity, similar policies are merged by aggregating support and label skew, and runtime statistics are carried over only for surviving broad policy text. A broad policy is treated as near a conflict-heavy region when its embedding has cosine similarity at least $0.85$ to a mixed-label conflict entry. Low-confidence broad items remain stored but are not surfaced unless their confidence later exceeds the label-specific threshold.

The two channels play different roles at runtime. Broad policies provide reusable default guidance induced from sparse failures. Local rules act as narrow warning or exception cues in regions where nearby cases split labels. Both are retrieved by semantic similarity. Broad policies are confidence-filtered before serialization, while retrieved local rules are serialized directly as narrow cues.

\subsection{Baselines implementation}
\label{app:baseline-prompting}

All baselines use the same binary decision task and JSON label interface as the base guardrail. We do not reproduce every baseline prompt verbatim, since the primary experimental distinction across methods is the memory object each baseline maintains and retrieves. Where a baseline was originally designed for a different supervision regime, we describe the adaptation explicitly below. Pure Prediction uses the base guardrail prompt alone. AGrail retrieves adaptive checklist-style notes and generates a short runtime checklist before classification. Synapse retrieves similar past decision records as case-level exemplars. ReasoningBank retrieves reusable natural-language memories induced from prior decision outcomes.

Synapse and ReasoningBank were originally designed for long-horizon agent tasks rather than single-step binary guardrail decisions, so we adapt them to the feedback interface studied in this work. For Synapse, trajectory exemplars are replaced by guardrail decision records consisting of the scenario, the model prediction, and the corrected label. This gives a case-memory baseline that tests whether adaptation can be achieved by retrieving similar reported failures directly, without inducing policy-level abstractions.

For ReasoningBank, the original trajectory-level memory construction is not directly applicable in our setting. ReasoningBank induces reasoning memories from success and failure trajectories, and in long-horizon tool-use tasks can extract multiple reflections from different parts of a trajectory, including individual tool-call decisions. By contrast, our deployment-time guardrail interface provides each adaptive method only with sparse misclassified cases and their corrected \textsc{allow}/\textsc{refuse} labels; it does not expose successful trajectories, dense labels for the full stream, or tool-call-level supervision. Instantiating the full ReasoningBank pipeline would therefore require additional information unavailable to the other baselines. We consequently implement ReasoningBank as a ReasoningBank-style broad-policy memory baseline: reported failures are converted into reusable natural-language policy memories and retrieved as decision-boundary guidance inside the same binary classification prompt. We do not add \textsc{LiSA}'s conflict-aware local memory or confidence-gated broad-policy surfacing to this baseline. This implementation matches the broad-policy-only role of \textsc{LiSA} without local rules or confidence-gated surfacing, making it both an external memory baseline adapted from prior work and a controlled reference point for isolating the contribution of \textsc{LiSA}'s two additional mechanisms.

Under this adaptation, the baselines form a natural memory ablation hierarchy: Pure Prediction removes memory entirely, Synapse tests raw case reuse, ReasoningBank tests broad policy abstraction under the same sparse failure-report interface, and \textsc{LiSA} tests whether adding conflict-aware local memory and conservative confidence-gated broad-policy reuse improves over broad abstraction alone.

\subsection{\textsc{LiSA} prompt templates}
\label{app:prompts}

\textsc{LiSA} uses fixed templates. The two templates below are the only LLM prompts introduced by \textsc{LiSA} beyond the shared base guardrail decision prompt. Indexed fields such as \texttt{Case \{i\}}, \texttt{Local Rule \{i\}}, \texttt{Memory \{i\}}, and \texttt{Failure \{i\}} denote repeated blocks; each prompt contains all retrieved or reported items up to the limits described above. Confidence updates are not prompted: after a surfaced broad item receives feedback, its support or contradiction count is updated by exact label match, and the Beta lower credible bound is recomputed. Existing policies are not rewritten by a separate LLM update prompt during refresh; newly reported failures are converted into additional structured items using Template 2, while policy merging and local-rule construction are deterministic procedures over embeddings, label counts, and stored policy metadata.

The inference template is instantiated with the current scenario, retrieved similar cases, retrieved local rules, and confidence-filtered broad policy items. If no local rule is retrieved and no broad item survives filtering, this template is not used and the system falls back to the base guardrail.

% --- Template 1 ---
\begin{tcblisting}{
title=\textbf{Prompt Template 1: \textsc{LiSA} Inference with Applicability Checks},
colback=white,
colframe=gray!60!black,
fonttitle=\bfseries,
arc=2mm,
listing only,
listing options={basicstyle=\small\ttfamily,breaklines=true,columns=fullflexible}
}
Relevant similar past cases:
  Case {i}: {case_summary}
  Outcome: {case_label}

Cases with semantic similarity despite conflicting labels:
  Content {i}: {canonical_policy}

Structured preventive memory:
  Memory {i}:
    Title: {title}
    Description: {description}
    Content: {content}
    Type: {rule_type}, label={recommended_label}, confidence={confidence}

Given the following new scenario, determine whether the action is
appropriate or inappropriate from a safety or privacy perspective. 
When applicable, prioritize high-confidence structured preventive 
memories over lower-confidence or less specific retrieved context.

Before the final answer, explicitly check each retrieved memory item for
applicability. For each item, decide one of: apply, skip, override, or
uncertain. Use apply when the item clearly matches the case. Use skip
when it is not a real fit. Use override only when a narrow exception-level
cue should dominate a broader rule. Use uncertain when the fit is
ambiguous. Then make the final decision using only the items that
survived this applicability check.

Scenario:
{scenario_text}

Respond with JSON only in this order:
{"reasoning": "briefly list memory applicability decisions, then final rationale",
 "label": "appropriate" or "inappropriate"}
\end{tcblisting}

The offline policy-induction template converts day-end failures into compact structured memory. The same template is used for broad preventive items; local rules are then regenerated from mixed-label neighborhoods during the memory refresh.

% --- Template 2 ---
\begin{tcblisting}{
title=\textbf{Prompt Template 2: Offline Preventive Memory Induction},
colback=white,
colframe=gray!60!black,
fonttitle=\bfseries,
arc=2mm,
listing only,
listing options={basicstyle=\small\ttfamily,breaklines=true,columns=fullflexible}
}
Reported failures to convert into preventive memory:

  Failure {i}:
    Scenario: {scenario_text}
    Model prediction: {predicted_label}
    Correct label: {true_label}

Convert these failures into concise, generalizable preventive memory items.
Avoid quoting case-specific names or wording. Keep each item compact and
reusable. Produce at most 3 items.

Each item must follow exactly this multiline format:
Title: short risk pattern title
Description: one-line summary
Content: preventive rule, decisive boundary, and when to defer if needed
Recommended label: appropriate or inappropriate
Rule type: general_policy or local_exception

Respond with JSON: {"insights": ["item1", "item2", ...],
"policies": ["item3", ...]}
\end{tcblisting}

The \texttt{Rule type} field in Template 2 is metadata attached to LLM-induced preventive items. \textsc{LiSA}'s conflict-aware local rules are constructed separately from mixed-label neighborhoods rather than generated by this prompt. For clarity, the deterministic local-rule rendering format is shown below.

% --- Renderer 1 ---
\begin{tcblisting}{
title=\textbf{Deterministic Template for Conflict-Aware Local Rule},
colback=white,
colframe=gray!60!black,
fonttitle=\bfseries,
arc=2mm,
listing only,
listing options={basicstyle=\small\ttfamily,breaklines=true,columns=fullflexible}
}
Local rule: In the boundary-heavy region '{region}', {recommended_outcome}.
Evidence: support={support_count}, nearby contradictions={contradict_count}.
Decisive pivots: {pivot_1}; {pivot_2}.
Use this as a narrow exception-level cue when the current case matches the
same local pattern.
\end{tcblisting}

\subsection{Existing assets and licenses}
\label{app:assets}

We use publicly available datasets and commercial API models. Table~\ref{tab:assets} summarizes access paths and licensing terms. We cite the original papers in the main text and use each asset in accordance with its released license or API terms of service. For \texttt{Claude-Haiku-4.5}, we use this model through Google Vertex AI. AgentHarm restricts use to research that improves the safety and security of AI systems; our use of the benchmark for evaluating safety guardrails is consistent with this restriction.

\begin{table}[h]
\centering
\small
\caption{\textbf{Assets used in this work.} 
Datasets and models with their access paths and license/terms.}
\label{tab:assets}
\vspace{3mm}
\setlength{\tabcolsep}{5pt}
\renewcommand{\arraystretch}{1.15}
\begin{tabular}{llll}
\toprule
\textbf{Asset} & \textbf{Type} & \textbf{Access} & \textbf{License / Terms} \\
\midrule
PrivacyLens~\citep{privacylens}     & Dataset   & GitHub  & CC BY 4.0 \\
ConFaide~\citep{confaide}           & Dataset   & GitHub  & CC BY 4.0 \\
AgentHarm~\citep{agentharm}         & Dataset   & HuggingFace             & MIT (with safety-use restriction) \\
Gemini-3.1-pro~\citep{gemini3}      & Model API & Google Vertex AI        & Vertex AI ToS \\
Gemini-3-flash~\citep{gemini3}      & Model API & Google Vertex AI        & Vertex AI ToS \\
Gemini-3.1-flash-lite~\citep{gemini3} & Model API & Google Vertex AI      & Vertex AI ToS \\
Gemini-embedding-001~\citep{geminiembedding}                & Model API & Google Vertex AI        & Vertex AI ToS \\
Claude-Haiku-4.5~\citep{claude4}    & Model API & Google Vertex AI           & Anthropic Usage Policy \\
\bottomrule
\end{tabular}
\end{table}

\section{Additional discussion and case study}

\subsection{Offline adaptation cost}
\label{app:offline-cost}

The cost--performance analysis in Section~\ref{sec:cost-performance} reports runtime latency per guarded decision because this cost is paid on every deployment input. \textsc{LiSA} additionally performs offline memory refresh when user-reported failures are accumulated. This cost is not on the critical inference path and can be batched across reports.

In our implementation, the offline policy-induction step consumes on average 427 input tokens and 2300 output tokens per reported failure. If a report-derived policy is reused over $K$ subsequent guarded decisions, its amortized offline cost is
\[
427/K \quad \text{input tokens}
\qquad\text{and}\qquad
2300/K \quad \text{output tokens}
\]
per decision. For example, even at $K=100$, this corresponds to only 4.27 input tokens and 23 output tokens per decision.

Thus, offline adaptation has a different cost structure from model scaling. Scaling the base guardrail increases cost on every inference call, whereas \textsc{LiSA} pays a small offline cost only when sparse failures are reported and then reuses the induced policies across many later inputs. This is why the main cost--F1 comparison focuses on runtime serving cost.

\subsection{Impact of offline manager quality}
\label{app:offline-manager-quality}

A natural question is whether the offline adaptation cost can be reduced by
using a smaller model as the offline policy manager. This changes a different
cost axis from the one measured in Section~\ref{sec:cost-performance}. Runtime
latency is incurred on every guarded input and is determined by the online
guardrail call, retrieval, and the serialized memory prompt. By contrast, the
offline manager is invoked only when sparse user-reported failures are
accumulated and memory is refreshed. Thus, using a weaker offline manager can
reduce intermittent refresh cost, but it does not directly reduce the
per-query critical inference path once the memory has been deployed.

\begin{wraptable}{r}{0.45\textwidth}
\vspace{-5mm}
\centering
\small
\caption{Impact of offline manager quality on final-day macro-F1. Same setting as Table~\ref{tab:component-ablation}: $\rho=0\%$, averaged across benchmarks, with Gemini-3.1-flash-lite as the online guardrail.}
\label{tab:offline_ablation}
\begin{tabular}{lc}
\toprule
\textbf{Offline Manager} & \textbf{Macro-F1} \\
\midrule
Gemini-3.1-pro (default) & 0.962 \\
Gemini-3-flash & 0.941 \\
Gemini-3.1-flash-lite & 0.915 \\
\bottomrule
\end{tabular}
\vspace{-3mm}
\end{wraptable}

Table~\ref{tab:offline_ablation} evaluates this trade-off by replacing the
default offline manager with a weaker model while keeping the online guardrail
fixed. Performance drops substantially, from 0.962 to 0.915 final-day macro-F1.
This indicates that the offline manager performs substantive abstraction work:
under sparse feedback, it must induce broad policies from few reports and
identify boundary-sensitive local refinements in mixed-label regions, rather
than merely formatting reported failures into memory entries. A weaker manager
can produce brittle, underspecified, or overly broad policies, which are then
reused across many subsequent online decisions.

This makes offline manager quality especially important in the deployment
regime studied here. If dense feedback were available after every trial, one
could update memory online and correct low-quality policies quickly. In the
guardrail setting, however, feedback arrives sparsely and memory is refreshed
only periodically. A single low-quality offline abstraction can therefore
persist and affect many future inferences. Since the offline refresh cost is
outside the critical serving path and is amortized over later guarded decisions
(Appendix~\ref{app:offline-cost}), we recommend using a strong offline manager
by default while keeping the online guardrail lightweight. Weaker offline
managers remain a viable option when refresh-time compute is constrained, but
Table~\ref{tab:offline_ablation} should be read as the expected quality cost of
that choice.

\subsection{Examples of generated policies}
\label{app:policy-examples}

In this section, we provide concrete examples of policy items generated by \textsc{LiSA} during the deployment simulation. These examples illustrate how the system induces broad general policies from sparse reports, as well as how it constructs conflict-aware local rules to handle mixed-label boundaries. Unnecessary metadata (such as internal tracking counts and IDs) have been omitted for clarity.

\begin{tcolorbox}[colback=white,colframe=gray!60!black,title=\textbf{Example 1: Broad General Policy},arc=2mm]
\textbf{Title:} Sharing General or Public Information \\
\textbf{Description:} Professionals can share non-sensitive, general information regardless of their confidential roles. \\
\textbf{Type:} \texttt{general\_policy} \hfill \textbf{Recommended Label:} \texttt{appropriate} \\
\textbf{Evidence:} Support = 170, Contradictions = 2 \hfill \textbf{Confidence:} 0.964
\tcblower
\textbf{Content:} It is appropriate for individuals in confidential roles to share general, publicly available, or non-sensitive information (e.g., event schedules, industry conferences, general techniques) with external parties, provided no protected personal or proprietary data is disclosed.
\end{tcolorbox}

\begin{tcolorbox}[colback=white,colframe=gray!60!black,title=\textbf{Example 2: Broad General Policy},arc=2mm]
\textbf{Title:} Conflation of Secrets with Public/General Info \\
\textbf{Description:} Do not confuse a subject's confidential secret with their public activities or general opinions. \\
\textbf{Type:} \texttt{general\_policy} \hfill \textbf{Recommended Label:} \texttt{appropriate} \\
\textbf{Evidence:} Support = 47, Contradictions = 2 \hfill \textbf{Confidence:} 0.879
\tcblower
\textbf{Content:} When evaluating a scenario, isolate the exact piece of information being shared. If the information is public advocacy, volunteer work, event attendance, or a general opinion, it is appropriate to share. Do not label it inappropriate just because the subject also has a related, highly sensitive secret in the background context. The boundary is whether the secret itself is being disclosed.
\end{tcolorbox}

\begin{tcolorbox}[colback=white,colframe=gray!60!black,title=\textbf{Example 3: Conflict-Aware Local Rules},arc=2mm]
\textit{These two complementary rules demonstrate how \textsc{LiSA} resolves a boundary-heavy region where cases split between appropriate and inappropriate contexts.}

\vspace{2mm}
\textbf{Local Rule 3a: Appropriate Cue} \\
\textbf{Content:} In the boundary-heavy region \textit{`sharing Ben attended a public lecture...'}, treat the case as \texttt{appropriate}. \\
\textbf{Decisive Pivots:}
\begin{itemize}[leftmargin=*,nosep]
    \item \textit{Information type differs:} appropriate cases leaned toward public lectures and interfaith dialogue, while inappropriate cases leaned toward fringe groups with radical beliefs.
    \item \textit{Sharing reason differs:} Emily expressed concern vs. Sarah shared her growing fears.
\end{itemize}

\tcblower

\textbf{Local Rule 3b: Inappropriate Cue} \\
\textbf{Content:} In the boundary-heavy region \textit{`sharing Ben attended a public lecture...'}, treat the case as \texttt{inappropriate}. \\
\textbf{Decisive Pivots:} Inverse of the pivots above, capturing the opposing boundary conditions.
\end{tcolorbox}

\end{document}